\documentclass{article}

\usepackage{microtype}
\usepackage{graphicx}
\usepackage{subcaption}
\usepackage{float}
\usepackage{booktabs} 

\usepackage{hyperref}
\usepackage{xurl}
\usepackage{tcolorbox}

\usepackage[preprint]{icml2026}

\usepackage{amsmath}
\usepackage{amssymb}
\usepackage{mathtools}
\usepackage{amsthm}

\usepackage[capitalize,noabbrev]{cleveref}

\theoremstyle{plain}
\newtheorem{theorem}{Theorem}[section]

\newtheorem{lemma}[theorem]{Lemma}

\theoremstyle{definition}

\theoremstyle{remark}

\usepackage[disable,textsize=tiny]{todonotes}

\icmltitlerunning{Democratic Preference Alignment via Sortition-Weighted RLHF}

\begin{document}

\twocolumn[
  \icmltitle{Democratic Preference Alignment via Sortition-Weighted RLHF}



  \icmlsetsymbol{equal}{*}

  \begin{icmlauthorlist}
    \icmlauthor{Suvadip Sana}{equal,cornell}
    \icmlauthor{Jinzhou Wu}{equal,cornell}
    \icmlauthor{Martin T. Wells}{cornell}
  \end{icmlauthorlist}

  \icmlaffiliation{cornell}{Cornell University}

  \icmlcorrespondingauthor{Martin T. Wells}{mtw1@cornell.edu}

  \icmlkeywords{AI Alignment, Reinforcement Learning from Human Feedback, RLHF, Democratic AI, Sortition, Preference Optimization, Fairness, Social Choice, Bradley Terry}

  \vskip 0.3in
]



\printAffiliationsAndNotice{\icmlEqualContribution}

\begin{abstract}
Whose values should AI systems learn? Preference-based alignment methods like RLHF derive their training signal from human raters, yet these rater pools are typically convenience samples that systematically over-represent some demographics and under-represent others. We introduce Democratic Preference Optimization (DemPO), a framework that applies algorithmic sortition, the same mechanism used to construct citizen assemblies, to preference-based fine-tuning. DemPO offers two training schemes: Hard Panel, which trains exclusively on preferences from a quota-satisfying mini-public sampled via sortition, and Soft Panel, which retains all data but reweights each rater by their inclusion probability under the sortition lottery. We prove that Soft Panel weighting recovers the expected Hard Panel objective in closed form. Using a public preference dataset that pairs human judgments with rater demographics and a 75-clause constitution independently elicited from a representative U.S. panel, we evaluate Llama models (1B–8B) fine-tuned under each scheme. Across six aggregation methods, the Hard Panel consistently ranks first and the Soft Panel consistently outperforms the unweighted baseline, with effect sizes growing as model capacity increases. These results demonstrate that enforcing demographic representativeness at the preference-collection stage, rather than post-hoc correction, yields models whose behavior better reflects values elicited from representative publics.
\end{abstract}
\section{Introduction}

Preference-based alignment methods, such as reinforcement learning from human feedback (RLHF) \cite{christiano2017deep} and related objectives, are widely used to shape the behavior of the model for deployment. These pipelines implicitly define ``whose values'' the system learns through the composition of the rater pool and the aggregation rule used to convert feedback into an optimization signal. In practice, rater pools are often convenience samples that can be demographically skewed relative to a target population, which can systematically overweight some groups' judgments and underweight others. This raises a governance question for aligned systems deployed at scale: how should we construct the feedback signal so it better reflects the values of the population the system is meant to serve \cite{gabriel2020artificial,conitzer2024social}?

We introduce \emph{Democratic Preference Optimization (DemPO)}, a sortition-based framework for constructing preference-learning objectives from demographically representative mini-publics. DemPO uses algorithmic sortition to define a lottery over quota-feasible panels \cite{flanigan2021fair}. From this lottery we obtain (i) a sampled panel that acts as the training electorate (Hard Panel), and (ii) per-rater inclusion probabilities that define a representativeness-aware weighting scheme over the full dataset (Soft Panel). Both schemes can be instantiated with standard preference objectives such as Direct Preference Optimization (DPO) \cite{rafailov2023direct}.

We evaluate DemPO in PRISM \cite{kirk2024prism}, which provides preference data and demographic data for the rater, and measure downstream alignment using a 75-clause constitution derived from a representative panel from the US \ \cite{huang2024collective}. Across six aggregation families (Bradley--Terry, Plackett--Luce, Borda, Copeland, Kemeny-Young, Mallows), Hard Panel and US-Rep rank above Soft and Full PRISM, while Soft consistently improves over Full PRISM. We also analyze how these effects grow with model size.

\subsection{Our Contributions}

\begin{figure*}[t]
  \centering
  \includegraphics[width=\textwidth]{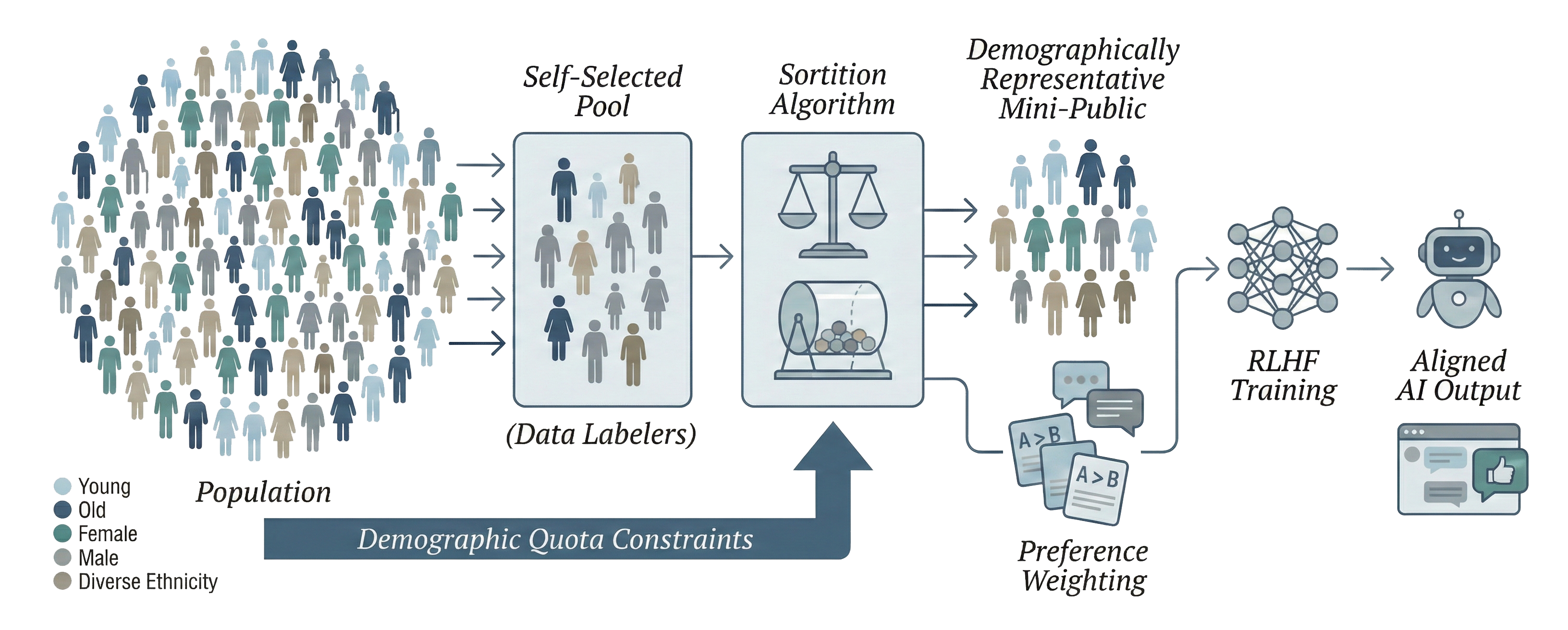}
  \caption{\textbf{The DemPO pipeline for democratic preference alignment.} A biased, self-selected pool of data labelers is transformed into a demographically representative mini-public via algorithmic sortition subject to population-derived quota constraints. Preferences from this representative panel (Hard Panel) or selection-probability-weighted preferences from all raters (Soft Panel) are then used for RLHF training, yielding AI systems aligned with broader public values.}
  \label{fig:pipeline}
\end{figure*}

We make four contributions:
\begin{itemize}
    \item \textbf{Sortition-weighted preference learning.} We introduce DemPO, a framework that operationalizes marginal demographic representativeness in preference-based post-training via quota-constrained algorithmic sortition \cite{flanigan2021fair}.
    \item \textbf{Hard and Soft Panel objectives.} We propose two training schemes: Hard Panel training on a single sampled mini-public, and Soft Panel training that weights each rater by their sortition inclusion probability. We give simple identities linking Soft Panel weighting to the expected Hard Panel objective.
    \item \textbf{Empirical evaluation of representative constitutional criteria.} Using PRISM \cite{kirk2024prism} and a constitution elicited from a representative panel of the United States \ \cite{huang2024collective}, we evaluate whether representativeness-aware training shifts the behavior of the model towards the panel-elicited criteria under multiple rank aggregation methods.
    \item \textbf{Model-size scaling and diagnostics.} We quantify how panel-based gains scale with model size and report update-magnitude diagnostics to ensure fair optimization budgets across conditions.
\end{itemize}

\subsection{Related Work}

\paragraph{Preference-based alignment and RLHF.}
RLHF \cite{christiano2017deep} aligns the models by optimizing a learned reward signal derived from human judgments. Many modern pipelines use direct preference objectives that avoid explicit reward-model training while still leveraging preference data, including DPO \cite{rafailov2023direct}. Constitutional AI \cite{bai2022constitutional} provides an alternative route that uses a written set of principles to guide the behavior of the model, often combined with preference learning. Our work is compatible with these methods: DemPO modifies how human feedback is selected or weighted before it is fed into a preference objective. Importance weighting is a standard technique in off-policy RL and has been applied to RLHF to correct for distribution shift between data collection and policy optimization \cite{munos2016safe,precup2000eligibility}; our soft-panel weighting can be viewed through this lens, where the weights encode demographic representativeness targets rather than policy mismatch.

\paragraph{Whose preferences and social choice perspectives.}
A growing body of work frames alignment as a problem of aggregation and collective decision-making, including questions about whose values should be represented and how to reconcile competing preferences \cite{gabriel2020artificial,conitzer2024social}. Complementing this normative framework, empirical analyzes study the’ language models of opinions of groups that tend to reflect \cite{santurkar2023whose}. Relatedly, \citet{buyl2025ai} argue that preference-based alignment inevitably grants annotators substantial alignment discretion in judging which outputs are ``better'' or ``safer,'' and propose metrics to measure when this discretion is exercised. Our contribution is an operational mechanism grounded in a standard governance tool: sortition-based mini-publics. We do not attempt to solve preference aggregation in full generality; instead, we study whether enforcing demographic representativeness in the feedback signal measurably shifts learned behavior toward criteria elicited from a representative panel \cite{huang2024collective}.

\paragraph{Democratic and diverse feedback collection.}
Democratic AI proposals have explored human-in-the-loop governance where decisions are shaped by group deliberation or majority vote \cite{koster2022human}. Other work studies preference heterogeneity directly by fine-tuning models to generate statements that maximize expected approval for a diverse group of people \cite{bakker2022agreement}. Related efforts emphasize collecting diverse feedback or encouraging diverse model outputs to improve robustness and reduce mode collapse \cite{lanchantin2025diverse}. DemPO targets a different failure mode: biased participation in feedback collection. We enforce representativeness at the level of \emph{who gets to contribute} to the optimization signal, either by selecting a quota-feasible panel or by reweighting all raters by their inclusion probabilities under the same sortition rule \cite{flanigan2021fair}.

\section{Methodology}
\label{sec:democratic-objectives}

Figure~\ref{fig:pipeline} summarizes the end-to-end pipeline from demographic targets to sortition-weighted preference optimization and evaluation.
We implement algorithmic sortition using the LEXIMIN procedure of \citet{flanigan2021fair}, which constructs a lottery over quota-feasible panels. Concretely, among all lotteries over feasible panels that satisfy the quota constraints, LEXIMIN chooses a lottery whose induced inclusion-probability vector $\pi=(\pi_i)_{i\in\mathcal{I}}$ is lexicographically maximal after sorting in ascending order (i.e., it maximizes the minimum inclusion probability, then the second-smallest, and so on). In this paper, we rely on two resulting properties that mirror standard citizens'-assembly intuitions: (i) \textbf{max--min protection} against near-zero inclusion probability for individuals in scarce demographic categories, and (ii) \textbf{symmetry within equivalence classes}, meaning that if two raters are indistinguishable with respect to the quota-relevant attributes, the procedure assigns them equal inclusion probabilities rather than arbitrarily favoring one over the other.

We formalize our setting using three levels: a target population, a pool of raters, and a panel produced via sortition.

\paragraph{Population, pool, and preference data.}
Let $\mathcal{P}$ denote the target population (e.g., U.S.\ adults), described by a set of demographic attributes
\[
A = \{A^{(1)},\dots,A^{(d)}\},
\]
where each attribute $A^{(t)}$ has a (finite) category set $\mathcal{A}^{(t)}$ (e.g., gender categories, age buckets, etc.).
For each attribute, we specify a \emph{marginal} population distribution
\[
p_{\text{pop}}^{(t)}(a) \in [0,1],
\; a \in \mathcal{A}^{(t)},
\; \sum_{a \in \mathcal{A}^{(t)}} p_{\text{pop}}^{(t)}(a) = 1.
\]
Importantly, we do \emph{not} assume a joint distribution over the full cross-product of categories; our quota constraints are enforced on these marginals (one attribute at a time).
We observe a finite \emph{pool} of raters
\[
\mathcal{I} = \{1,\dots,n\},
\]
where each rater $i \in \mathcal{I}$ has a demographic attribute vector
\[
a_i = \big(a_i^{(1)},\dots,a_i^{(d)}\big),
\; a_i^{(t)} \in \mathcal{A}^{(t)}.
\]
This pool is typically a biased self-selected subset of the population.
In our PRISM instantiation we further restrict $\mathcal{I}$ to raters who appear in preference data (PRISM utterances), so that sortition is run on the same pool that can contribute to the training objective; survey-only respondents with no preference data are excluded. Any rater with zero usable comparisons after pre-processing contributes nothing to $\mathcal{D}$.

From this pool, we obtain a dataset of pairwise preference comparisons
\[
\mathcal{D} = \{(x_j, y_j^+, y_j^-, r_j)\}_{j=1}^N,
\]
where $x_j$ is a prompt or context, $y_j^+$ and $y_j^-$ are two candidate responses, and $r_j \in \mathcal{I}$ is the rater who prefers $y_j^+$ to $y_j^-$. We consider preference-learning losses of the form
\[
\ell(\theta; x, y^+, y^-) \in \mathbb{R}_{\ge 0},
\]
where $\theta$ are the parameters of the conditional language model $\pi_\theta(\cdot \mid x)$. For example, under Direct Preference Optimization (DPO) the loss may be
\begin{align}
\ell_{\mathrm{DPO}}(\theta;x,y^+,y^-)
&= -\log \sigma\Big(
\beta \big[
\log \pi_\theta(y^+|x) \\
&- \log \pi_\theta(y^-|x) \nonumber - c(x,y^+,y^-)
\big] \Big),
\end{align}
where $\sigma$ is the logistic function, $\beta > 0$ is a temperature and $c(\cdot)$ is a baseline term derived from a reference policy.

\paragraph{Multi-turn contexts and pair construction (PRISM instantiation).}
While the abstract definition above allows $x_j$ to be any prompt/context, our implementation uses multi-turn conversational contexts.
For each conversation thread and turn $t$, we construct $x_t$ as the full message history up to that point by concatenating user messages with a single assistant message per prior turn.\footnote{To avoid inconsistent branching histories, we include only the chosen response for each past turn; if a turn has no explicit chosen marker, we fall back to the highest-scoring response.}
At the current turn $t$, PRISM provides multiple candidate model completions with scalar scores; we generate pairwise comparisons by taking all ordered pairs whose score difference exceeds a threshold $\delta$ (we use $\delta=0$ in experiments), yielding many comparisons per turn while keeping the history coherent.
We interpret PRISM scores only within the same conversation turn (same rater and context) and do not assume calibration across different turns or threads; setting $\delta=0$ includes even small score gaps, which may be noisy, but provides a high-coverage training signal and can be tightened in future work by increasing $\delta$.
As a practical data-quality safeguard, we drop placeholder values such as ``EMPTY STRING'' before constructing contexts or preference pairs.

\paragraph{Sortition over panels.}
A \emph{panel} is any subset $S \subseteq \mathcal{I}$ of raters of fixed size $|S| = k$ that satisfies a collection of demographic quota constraints. Concretely, for each attribute $t \in \{1,\dots,d\}$ and category $a \in \mathcal{A}^{(t)}$, we specify the lower and upper bounds
\[
L_{t,a} \le \#\{ i \in S : a_i^{(t)} = a\} \le U_{t,a},
\]
chosen so that the panel approximately satisfies the quota constraints within some tolerance. Let $\mathcal{S}$ denote the family of all such feasible panels.

Following work on algorithmic sortition \cite{flanigan2021fair}, we assume access to a lottery over panels constructed using the Sortition Foundation LEXIMIN procedure
\[
\pi_{\mathrm{panel}} : \mathcal{S} \to [0,1],
\; 
\sum_{S \in \mathcal{S}} \pi_{\mathrm{panel}}(S) = 1,
\]
constructed to satisfy the quota constraints while optimizing a fairness objective over individuals (e.g., maximizing the minimum selection probability). This lottery induces a \emph{selection probability} for each rater $i \in \mathcal{I}$,
\begin{equation}
\label{eq:selection-prob}
\pi_i \;:=\; \Pr_{S \sim \pi_{\mathrm{panel}}}[ i \in S ]
\;=\; \sum_{S \in \mathcal{S}: i \in S} \pi_{\mathrm{panel}}(S).
\end{equation}

We record two basic identities of these inclusion probabilities (including $\sum_i \pi_i = k$ and a link between Soft Panel weighting and the expected Hard Panel objective) in Appendix~\ref{app:basic-identities}.

We now define two training schemes that use $\pi_{\mathrm{panel}}$ and the induced $\{\pi_i\}$ to implement democratic objectives.

\subsection{Hard Panel Training}
\label{sec:hard-panel}

In the \emph{hard panel} scheme, we mirror the practice of convening a single mini-public by sortition.
We first draw one panel
\begin{equation}
\label{eq:sample-panel}
S \sim \pi_{\mathrm{panel}}(\cdot), \quad S \in \mathcal{S}, \quad |S| = k,
\end{equation}
and then restrict training to the preference data supplied by raters in $S$.

Formally, define the panel-restricted dataset
\begin{equation}
\label{eq:panel-dataset}
\mathcal{D}_S := \{ (x_j, y_j^+, y_j^-, r_j) \in \mathcal{D} : r_j \in S \}.
\end{equation}
The hard panel training objective is then
\begin{equation}
\label{eq:hard-objective}
\mathcal{L}_{\mathrm{hard}}(\theta \mid S)
\;=\;
\frac{1}{|S|}
\sum_{i \in S}
\frac{1}{N_i}
\sum_{j : r_j = i}
\ell\big(\theta; x_j, y_j^+, y_j^-\big).
\end{equation}

This objective treats the sampled panel as a \emph{democratic jury}: only panel members’ preferences
contribute to updates. Relative to training on the full rater pool, this can reduce bias induced by an
unrepresentative annotator population, but it also reduces the effective dataset size and can increase
optimization variance. In principle, this variance could be reduced by averaging across multiple
independently sampled panels or by explicitly optimizing an expectation over $S$; however, in our
main experiments we follow the citizen-assembly analogy and train on a single sampled panel.
The per-rater normalization by $N_i$ ensures each panel member contributes equally regardless of
how many comparisons they provided.

To connect hard-panel training to the soft weighting scheme below, define the per-rater mean loss
\[
L_i(\theta) \;:=\; \frac{1}{N_i}\sum_{j:r_j=i}\ell(\theta; x_j, y_j^+, y_j^-).
\]
Then \eqref{eq:hard-objective} is $\mathcal{L}_{\mathrm{hard}}(\theta\mid S)=\frac{1}{k}\sum_{i\in S}L_i(\theta)$. Taking expectations and using \eqref{eq:selection-prob} gives
\[
\mathbb{E}_{S \sim \pi_{\mathrm{panel}}}\!\left[\mathcal{L}_{\mathrm{hard}}(\theta\mid S)\right]
=
\frac{1}{k}\sum_{i=1}^n \pi_i\,L_i(\theta),
\]
which matches the soft objective \eqref{eq:soft-objective} when $w_i=\pi_i$ and $Z=\sum_i w_i=k$.

\subsection{Soft Panel Weighting}
\label{sec:soft-panel}

The hard panel scheme discards preference data from raters who are not selected into $S$.
As an alternative, we propose a \emph{soft} scheme that retains all preference comparisons but weights
each rater’s contribution according to their democratic selection probability $\pi_i$.

We define nonnegative per-rater weights
\begin{equation}
\label{eq:rater-weights}
w_i = f(\pi_i), \; i \in \mathcal{I},
\end{equation}
where $f : [0,1] \to \mathbb{R}_{\ge 0}$ is a monotone transformation.
Typical choices include $f(\pi_i)=\pi_i$ (direct weighting by inclusion probability), normalization by
total mass (equivalent here since $\sum_i \pi_i = k$), or clipped variants to control variance when
weights are highly concentrated.

The resulting \emph{soft democratic} training objective is
\begin{equation}
\label{eq:soft-objective}
\mathcal{L}_{\mathrm{soft}}(\theta)
\;=\;
\frac{1}{Z}
\sum_{i=1}^n
w_i
\frac{1}{N_i}
\sum_{j : r_j = i}
\ell\big(\theta; x_j, y_j^+, y_j^-\big),
\;
Z=\sum_{i=1}^n w_i.
\end{equation}
We apply two normalizations. First, we normalize each rater’s contributions by $N_i$, the number of
preference comparisons supplied by rater $i$ (typically $N_i > 1$ in datasets such as PRISM), so that
each rater contributes equally within the hard panel and proportionally to $w_i$ in the soft panel,
regardless of annotation volume. Second, to keep the \emph{optimization scale} comparable to
standard (unweighted) DPO under minibatch training, we compute the weighted loss as a
\emph{self-normalized weighted mean} within each minibatch:
\[
\widehat{\mathcal{L}}_{\mathrm{batch}}(\theta)
=
\frac{\sum_{b \in \mathcal{B}} \omega_b\,\ell_b(\theta)}{\sum_{b \in \mathcal{B}} \omega_b},
\qquad
\omega_b = \frac{w_{r_b}}{N_{r_b}},
\]
where $\ell_b(\theta)$ denotes the per-example preference loss for batch element $b$ and $\omega_b$
is the corresponding per-example weight (including the $1/N_i$ normalization for rater $r_b$). This
makes optimization invariant to global rescaling of all weights and empirically
stabilizes update magnitudes across training variants. In distributed training, we compute the
denominator using the sum of weights across all devices in the global batch.
Equivalently, the population objective in \eqref{eq:soft-objective} can be viewed as sampling a rater
$i$ with probability proportional to $w_i$ and then sampling a comparison uniformly from that
rater’s set of comparisons, so that the per-rater normalization by $N_i$ implements ``one person,
one voice'' while the rater weights encode representativeness targets.

With $w_i=\pi_i$, the soft objective matches the expected hard-panel objective and yields a
self-normalized estimator that can be interpreted as training over a mixture of many possible
mini-publics drawn from the same sortition lottery, while using all available data for lower-variance
updates.

The LEXIMIN sortition procedure defines the lottery $\pi_{\mathrm{panel}}$ over feasible panels; from
this lottery we obtain per-rater inclusion probabilities $\pi_i$ via \eqref{eq:selection-prob}. In
practice, $\{\pi_i\}$ can be computed exactly from the optimization solution underlying the LEXIMIN
sortition lottery (we do so in all primary experiments), or estimated by Monte Carlo sampling of panels
and taking empirical inclusion frequencies (used only for LEGACY/random baselines). We use
the same $\{\pi_i\}$ for both hard-panel sampling and soft-panel weighting to ensure that the two
schemes implement the same democratic target. We report selection-probability diagnostics and concentration statistics for $\{\pi_i\}$ in Appendix~\ref{app:soft-weights}.

\section{Experiments}
\label{sec:experiments}

We evaluate whether sortition-based democratic training changes downstream behavior compared to standard preference fine-tuning on the full pool of PRISM raters and the US-Rep subset.

\subsection{Data: PRISM Preferences and Demographics}
\label{sec:data-prism}
We use the PRISM Alignment dataset \cite{kirk2024prism}, which contains multi-turn conversations with model responses, human preferences/ratings, and rater demographics.
Demographic metadata of this granularity are rarely available in public preference datasets, making it difficult to study or enforce representativeness constraints in most RLHF settings; PRISM enables these analyses by pairing preference data with annotator attributes.
From PRISM we use: (i) \texttt{conversations.jsonl} and \texttt{utterances.jsonl} to build multi-turn preference pairs and (ii) \texttt{survey.jsonl} to obtain demographic attributes for sortition.
We filter the sortition pool to raters who appear in PRISM preference data (excluding survey-only respondents with no utterances), and we remove placeholder ``EMPTY STRING'' prompts/responses before constructing histories or comparisons.
Our pre-processing reconstructs multi-turn contexts at the turn level. For each conversation and turn $t$, we build
the context $x_t$ as the full history of user messages and the single chosen (or highest-scoring) assistant response
per prior turn, yielding a coherent trajectory; if a prior turn is missing in the conversation history, we fall back
to the turn-level prompt when available (best-effort history).
We then collect all the responses of the candidate model in turn $t$ and create
pairwise preferences for every ordered pair with a score difference that exceeds a threshold $\delta$ (using $\delta=0$).
This yields a multi-turn preference dataset that preserves conversational context (see Appendix~\ref{app:data-sizes}
for dataset sizes). The resulting training distribution differs across conditions, which can induce different update magnitudes even when the per-step loss scale and the total number of optimizer updates are matched; we quantify these effects in Appendix~\ref{app:update-magnitude}.

\subsection{Sortition Targets and Panel Configuration}
\label{sec:panel-config}
Our target population is U.S.\ adults. We define quota constraints over seven attributes: ethnicity, religion, marital status, education, employment status, gender, and age.
Target proportions are set from 2020-era U.S.\ population estimates: for attributes available in the Census/ACS (e.g., ethnicity, age, gender, education, employment status, and marital status) we use ACS estimates \cite{uscensus2020acs}, and for religious affiliation we use Pew Research Center estimates from a nationally representative survey \cite{pew2021unaffiliated}.
We use the panel size $k=145$ and a relative tolerance $\tau=0.1$, implemented as bounds $[p(1-\tau),\,p(1+\tau)]$ for each proportion of categories; $k=145$ is the largest panel size that remains feasible under these constraints.
To avoid infeasibility due to sparse or ambiguous survey responses, we treat ``Other'' and ``Prefer not to say'' as \emph{slack} categories with no quota bounds.
Unless otherwise noted, we do not filter the rater pool by study locale; instead, panels are selected from the full PRISM rater pool and constrained to match U.S.\ demographic targets. \footnote{In our extracted pairs, the resulting hard panel contains raters from multiple locales; see Appendix~\ref{app:limitations-locale}.} Additional feasibility checks and quota-satisfaction diagnostics are reported in Appendix~\ref{app:panel-feasibility}.

\subsection{Models and Training}
\label{sec:training-setup}
We train variants of \texttt{meta-llama/Llama-3.1-8B} \cite{aimeta2024llama31modelcard} (Base) using Direct Preference Optimization (DPO) \cite{rafailov2023direct}, and replicate the same procedure for \texttt{meta-llama/Llama-3.2-1B} and \texttt{meta-llama/Llama-3.2-3B} to study scaling behavior.
We deliberately start from the \emph{base} (non-instruction-tuned) checkpoint to isolate the effect of democratic preference alignment: instruction-tuned models already incorporate additional supervised alignment data, system prompts, and conversational behaviors that would confound attribution of gains to sortition-weighted preference training.
A caveat is that base checkpoints can be less fluent in chat-style interaction; accordingly, we evaluate all models with a shared chat template and interpret the results as relative changes induced by preference training rather than absolute chatbot quality.
We use full-parameter fine-tuning (no LoRA), bfloat16, and a frozen reference model initialized from the same base checkpoint.
We convert each PRISM interaction into a conversational format with multi-turn history and rely on the model's chat template during training (``User: \dots'' $\rightarrow$ ``Assistant: \dots'').
Across all runs we use $\beta=0.1$ and learning rate $5\!\times\!10^{-6}$, and we match the optimization budget by training each fine-tuned model for the same number of optimizer updates: $3538$ steps, which corresponds to $2$ epochs on the Full PRISM training set. For smaller datasets (e.g., Hard Panel) we resample training pairs with replacement so that all conditions and model sizes receive the same number of updates.

\paragraph{Training conditions.}
We compare five models:
\begin{itemize}
  \item \textbf{Base}: no fine-tuning.
  \item \textbf{Full PRISM}: DPO on all PRISM pairs.
  \item \textbf{US-Rep}: DPO on the PRISM subset flagged as representative of U.S.\ raters (PRISM's \texttt{included\_in\_US\_REP}); this is a stratified subset provided by PRISM rather than a sortition-based panel.
  \item \textbf{Hard Panel}: compute a lottery over feasible panels using the LEXIMIN sortition algorithm \cite{flanigan2021fair} (implemented as an optimization problem solved with Gurobi), sample one panel from this lottery, and train only on pairs from selected raters, with per-rater normalization by $N_i$.
  \item \textbf{Soft Panel}: compute per-rater selection probabilities $\{\pi_i\}$ induced by the same LEXIMIN lottery and weight each preference pair by the rater's $\pi_i$, with per-rater normalization by $N_i$ and per-minibatch self-normalization (weighted mean) to keep update scale comparable to unweighted DPO. Pairs with $\pi_i=0$ are dropped before training.
\end{itemize}

\begin{figure*}[t]
  \centering
  \includegraphics[width=\textwidth]{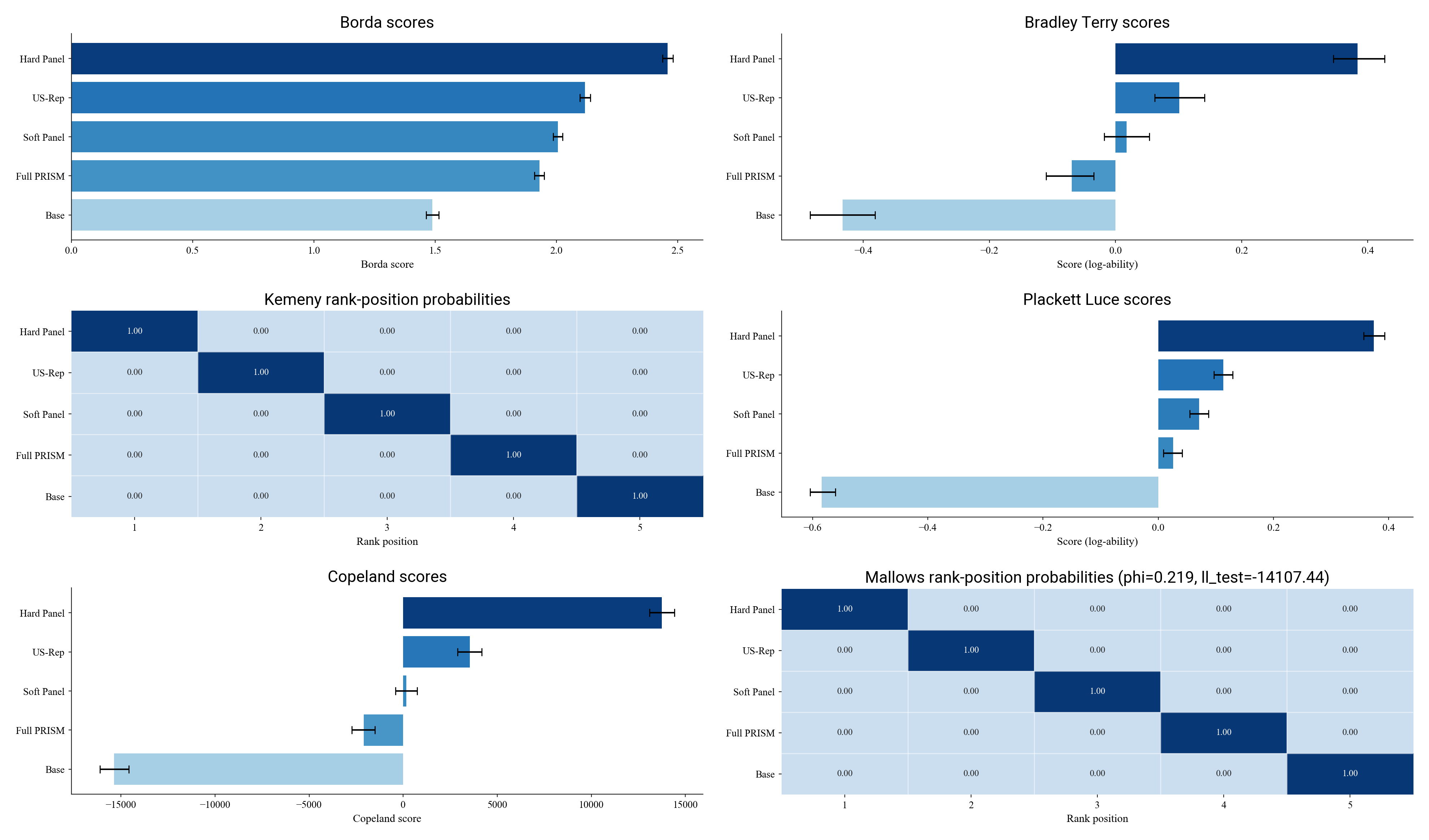}
  \caption{\textbf{Model ranking under multiple aggregation methods (Llama-3.1-8B).}
  Left: Borda and Copeland scores with 95\% bootstrap confidence intervals, and Kemeny consensus summarized as rank-position probabilities under bootstrap resampling.
  Right: Bradley--Terry and Plackett--Luce log-ability scores with 95\% bootstrap confidence intervals, and Mallows (Kendall) rank-position probabilities under bootstrap resampling (with fitted $\phi$ and held-out $\ell_{\text{test}}$).
  All bootstrap summaries use $n{=}1000$ resamples.
  Across BT, Borda, Copeland, Kemeny, and Mallows, Hard Panel ranks highest, US-Rep ranks second, Soft Panel ranks above the Full PRISM baseline, and the Base model is consistently worst.}
  \label{fig:score-grid-8b}
\end{figure*}

\subsection{Constitutional Evaluation Protocol}
\label{sec:eval-protocol}
We evaluate models against the collective constitutional principles introduced by \citet{huang2024collective}, which provides a 75-clause constitution derived from structured public input. In their study, approximately $1{,}000$ U.S.\ adults were recruited to be broadly representative across age, gender, income, and geography. The participants proposed and voted on candidate rules through an online interface, contributing statements $1{,}127$ and casting $38{,}252$ votes.
This constitution is a natural external evaluation target for our setting: our sortition constraints explicitly aim to approximate U.S.\ adult demographics, and the constitution aggregates normative preferences elicited from a representative U.S.\ public rather than from the PRISM rater pool. Using a public, deliberative process also reduces the risk that evaluation criteria simply mirror the idiosyncrasies of the PRISM rater pool.
For each clause, we generate $40$ standalone questions using GPT-5.2, producing $75\times 40 = 3000$ evaluation prompts.\footnote{The question-generation prompt explicitly forbids references to prior turns or missing context.}
Each candidate model answers every question using greedy decoding with a shared chat template, and we apply stop rules to prevent the model from continuing into additional dialogue turns.

\paragraph{LLM-as-a-judge.}
We use GPT-5.2 as an automated judge to rank all model responses listwise (best-to-worst) for each (clause, question) pair. Using strong LLMs as judges is a scalable proxy for human evaluation in open-ended settings and can achieve agreement with human preferences comparable to inter-human agreement on established benchmarks \cite{zheng2023judging}.
This approach is particularly appropriate here because the constitution specifies abstract normative principles rather than single “correct” answers, and listwise ranking provides a direct way to compare competing responses to the same prompt. We mitigate judge-model idiosyncrasies by collecting multiple independent rankings, randomizing response order on each call, and reporting agreement statistics and confidence intervals.
We collect $5$ independent judge rankings per question, randomizing the order of responses shown to the judge on each trial to mitigate position bias \cite{shi2025judging} (Appendix~\ref{app:ablations}).
This yields $3000\times 5 = 15{,}000$ complete rankings.
For pairwise analyses, we convert listwise rankings into pairwise wins and compute both (i) vote-level win rates (across all $5$ judge rankings) and (ii) majority winners per question; Bradley--Terry is fit to the resulting pairwise preferences (30,000 total in the 8B setting).

\subsection{Results}
\label{sec:results}

We score the resulting rankings using multiple rank aggregation models: Borda and Copeland scores \cite{borda1781memoire,copeland1951reasonable}, a Kemeny-Young consensus ranking \cite{kemeny1962mathematical,young1988condorcet}, Bradley--Terry (BT) fitted to pairwise preferences \cite{bradley1952rank}, Plackett--Luce (PL) fitted to listwise rankings \cite{plackett1975analysis,luce1959individual}, and a Mallows (Kendall) model with held-out likelihood evaluation \cite{mallows1957nonnull}.
We use $1000$ bootstrap resamples to compute confidence intervals for BT/PL and Borda/Copeland, and to quantify uncertainty in the Kemeny and Mallows rank-position summaries. For numerical reproducibility, raw score tables for each aggregation method are provided in Appendix~\ref{app:raw-scores}.

\paragraph{Main effect: democratic panels improve over the full-pool baseline.}
Across the scoring families in \cref{fig:score-grid-8b} (BT, Borda, Copeland, Kemeny, Mallows), Hard Panel ranks highest, US-Rep ranks second, and Soft Panel consistently improves over Full PRISM.
This suggests that enforcing demographic representativeness at training time, either by sampling a representative mini-public (Hard Panel) or by weighting all raters by their selection probability (Soft Panel), yields measurable improvements over standard preference training on the full rater pool.
US-Rep's strong performance indicates that demographic filtering alone can yield substantial gains over training on the full, unfiltered pool. However, Hard Panel's consistent improvement over US-Rep suggests that how representativeness is achieved matters, not just whether aggregate demographic targets are met. We hypothesize two mechanisms. First, LEXIMIN's max-min fairness property ensures that no individual in a scarce demographic intersection has near-zero inclusion probability, potentially providing better coverage of underrepresented viewpoints that a fixed stratified sample might miss. Second, LEXIMIN's symmetry property assigns equal inclusion probability to demographically equivalent individuals, reducing the influence of idiosyncratic annotator effects that could arise when a fixed subset is selected by other criteria (e.g., response rate or annotation volume). US-Rep, by contrast, is a pre-constructed stratified subset whose selection procedure may optimize for different goals. Disentangling these mechanisms—and testing whether the gap persists under alternative sortition algorithms—is a direction for future work.
Although the Hard Panel and US-Rep variants exhibit larger realized parameter updates under the matched budget $3538$, we verify that the ranking is not an artifact of “learning more” by retraining the Hard Panel and US-Rep for fewer steps to match the magnitude of Full/Soft update (Appendix~\ref{app:update-magnitude}). The ordering remains unchanged, while the score gaps shrink modestly (e.g., the Soft--Full Borda gap drops from $0.076$ to $0.066$ and the Hard--US-Rep gap drops from $0.341$ to $0.315$).

\paragraph{Scaling with model size.}
\begin{figure}[t]
  \centering
  \includegraphics[width=0.95\columnwidth]{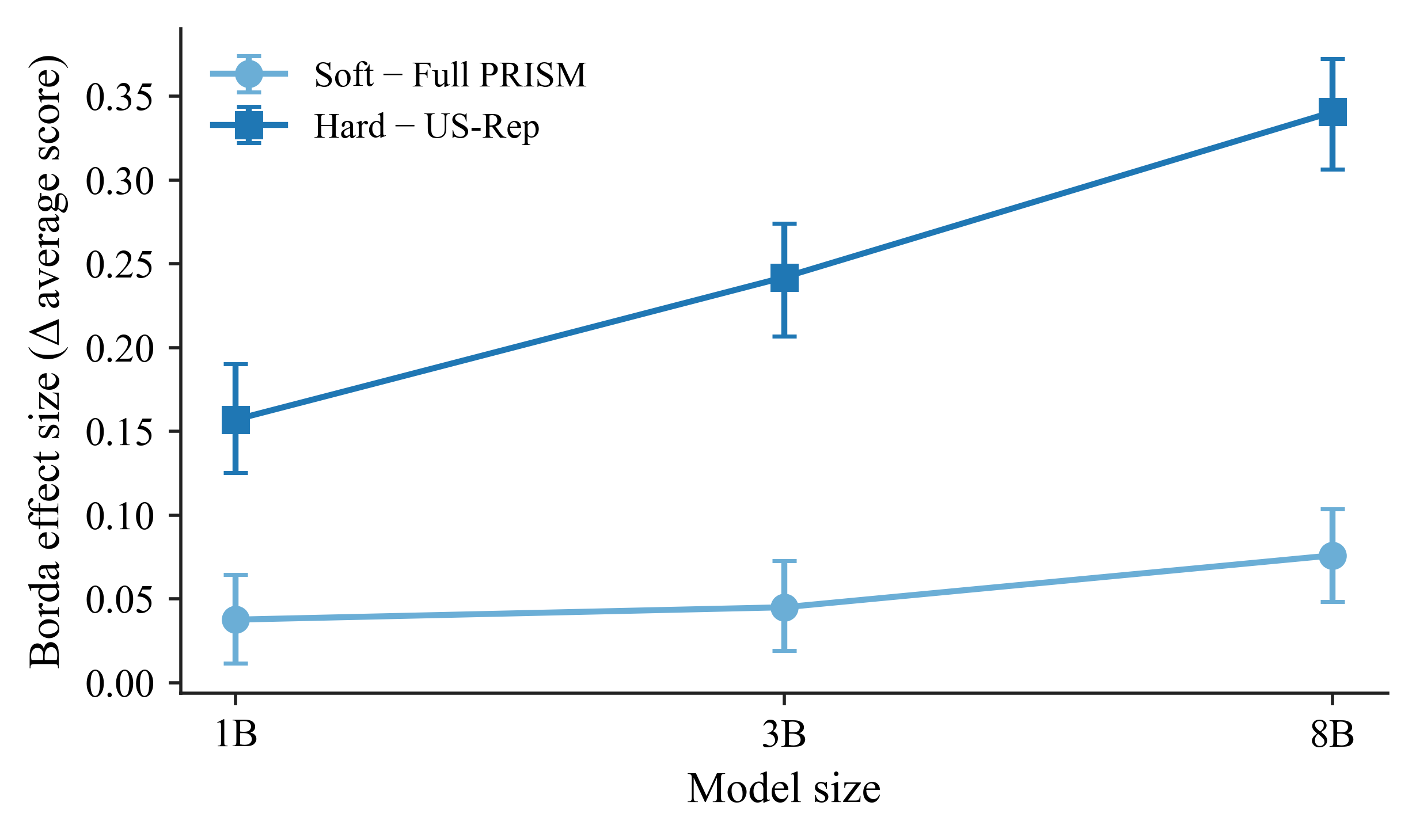}
  \caption{\textbf{Panel advantage grows with model size.} Effect sizes are computed from bootstrap resampling of listwise rankings (Borda average score differences). The Soft Panel vs. Full PRISM gap and the Hard Panel vs. US-Rep gap both increase from 1B to 3B to 8B. All error bars are 95\% bootstrap CIs ($n{=}1000$).}
  \label{fig:model-size-effect}
\end{figure}
We repeat the same evaluation protocol for Llama-3.2-1B and Llama-3.2-3B and observe the same ordering as the 8B results (Hard $>$ US-Rep $>$ Soft $>$ Full $>$ Base). Figure~\ref{fig:model-size-effect} summarizes effect sizes computed from bootstrap resampling of listwise rankings using Borda average score differences. Across these checkpoints (Llama-3.2 at 1B/3B and Llama-3.1 at 8B), the panel advantage grows with model size: the Soft--Full gap increases from $0.038$ (1B) to $0.045$ (3B) to $0.076$ (8B), and the average panel advantage (mean of Hard/Soft minus mean of Full/US-Rep) increases monotonically as well (Appendix~\ref{app:model-size}).

\paragraph{Judge reliability.}
To validate that the rankings are not dominated by noise from the judges, we measure the agreement between the $5$ judge rankings per question.
Across all questions, average inter-judge Kendall $\tau$ is $0.776$ and Fleiss' $\kappa$ on the \#1 choice is $0.710$, indicating substantial agreement.
We treat the resulting $15{,}000$ rankings as repeated measurements and use majority vote (or vote counts) in derived pairwise analyses.

\paragraph{Pairwise win rates.}
We compute question-level bootstrap win probabilities for every model against every baseline, reported as the vote-level win rate $\Pr(\text{model beats baseline})$ (bootstrap $n{=}1000$).
For Llama-3.1-8B, Soft beats Full with probability $0.520$ (95\% CI $[0.506,0.534]$), Hard beats Full with probability $0.618$ (95\% CI $[0.602,0.634]$), and Hard beats US-Rep with probability $0.576$ (95\% CI $[0.561,0.592]$).
This pairwise view complements the rank-aggregation picture in \cref{fig:score-grid-8b} by quantifying effect sizes in a common metric.

\paragraph{Limitations.}
Our U.S.\ demographic targets are applied to the PRISM rater pool, which includes raters from multiple locales, and some attributes (e.g., religion) require survey-based rather than Census-based population targets.
We treat this as a practical compromise to avoid infeasibility and report locale composition in Appendix~\ref{app:limitations-locale}.
More broadly, DemPO relies on preference datasets that record annotator demographics; PRISM is currently unusual in making such metadata available, and scaling these methods will require future preference data collection efforts to record (and govern) demographic attributes responsibly.

\section{Conclusion}
We introduced \emph{DemPO}, a democratic preference-optimization pipeline that uses algorithmic sortition to enforce demographic representativeness in preference-based post-training.
Using PRISM human preferences and demographic metadata, we instantiate two democratic training schemes: a \emph{Hard Panel} that trains on preferences from a representative mini-public sampled via sortition, and a \emph{Soft Panel} that reweights all preferences by the rater's selection probability under the same sortition lottery.
Across multiple model sizes and a constitution derived from a representative U.S.\ panel, Hard Panel and US-Rep rank above Soft and Full PRISM, while Soft consistently improves over Full PRISM under multiple aggregation families.

Overall, these results support the hypothesis that explicitly incorporating representativeness constraints into preference learning can shift aligned behavior toward values elicited from representative publics.
Important limitations include reliance on a single base model family, an automated judge, and practical compromises around locale filtering and census-incomplete attributes.
Future work should scale DemPO to larger models and human judging, evaluate locale-specific panels when sufficient data is available, and explore richer constraint families beyond marginal quotas.


\section*{Software and Data}
All scripts, configurations, and analysis code used in our experiments are provided in the accompanying repository.
Our evaluation questions and aggregate evaluation artifacts can be reproduced by running the pipeline described in Section~\ref{sec:experiments} and Appendix~\ref{app:reproducibility}.
Repository: \url{https://github.com/Wu-Jinzhou/democratic-llm}.


\section*{Impact Statement}
This work aims to make preference-based alignment more representative of the broader public by replacing ad hoc rater pools with panels (or weights) generated by a transparent, quota-constrained sortition procedure.
If deployed responsibly, such methods could reduce systematic misalignment that arises when feedback disproportionately reflects the values of a narrow, self-selected group, and could improve the legitimacy of alignment decisions for public-facing systems.

At the same time, demographic representativeness is not a complete notion of democratic legitimacy: quotas are imperfect proxies for values, population targets are approximate, and the resulting systems may still under-represent important viewpoints not captured by our attributes.
Moreover, using demographic data in model training and evaluation raises privacy and misuse concerns; we therefore rely only on de-identified rater IDs and public aggregate statistics, treat ambiguous categories as slack, and recommend strong governance and transparency when applying such techniques in practice.
Finally, LLM-as-a-judge evaluations can inherit biases from the judge model; we mitigate this with repeated judging, randomized response ordering, and diagnostic analyzes, but human oversight remains important before drawing high-stakes conclusions.

\bibliography{references}
\bibliographystyle{icml2026}

\newpage
\appendix
\onecolumn
\raggedbottom

\section{Additional Experimental Details}

\subsection{Basic Identities for Inclusion Probabilities}
\label{app:basic-identities}

\begin{lemma}[Inclusion probabilities sum to the panel size]
\label{lem:pi-sum-k}
Let $S \sim \pi_{\mathrm{panel}}$ be a fixed size random panel $|S|=k$, and let $\pi_i = \Pr[i\in S]$. Then $\sum_{i\in\mathcal{I}} \pi_i = k$.
\end{lemma}
\begin{proof}
Let $\mathbf{1}\{i\in S\}$ be the indicator that rater $i$ is included. Then $|S|=\sum_i \mathbf{1}\{i\in S\}=k$ almost surely. Taking expectations and using linearity,
$\sum_i \Pr[i\in S]=\sum_i \mathbb{E}[\mathbf{1}\{i\in S\}]=\mathbb{E}[|S|]=k$.
\end{proof}

\begin{lemma}[Soft weights recover the expected hard-panel objective]
\label{lem:soft-equals-expected-hard}
Let $L_i(\theta)$ denote the per-rater mean loss
\[
L_i(\theta) \;:=\; \frac{1}{N_i}\sum_{j:r_j=i}\ell(\theta;x_j,y_j^+,y_j^-),
\]
and recall the hard-panel objective for a fixed-size panel $|S|=k$,
\[
\mathcal{L}_{\mathrm{hard}}(\theta\mid S)=\frac{1}{k}\sum_{i\in S}L_i(\theta).
\]
Define the soft objective as a rater-weighted mean,
\[
\mathcal{L}_{\mathrm{soft}}(\theta)=\frac{1}{\sum_i w_i}\sum_{i=1}^n w_i\,L_i(\theta).
\]
Then with $w_i=\pi_i$ (and $\sum_i \pi_i = k$), we have
\[
\mathbb{E}_{S \sim \pi_{\mathrm{panel}}}\!\left[\mathcal{L}_{\mathrm{hard}}(\theta\mid S)\right]
=
\mathcal{L}_{\mathrm{soft}}(\theta).
\]
\end{lemma}
\begin{proof}
Taking expectations and using linearity,
\[
\mathbb{E}_{S}\!\left[\mathcal{L}_{\mathrm{hard}}(\theta\mid S)\right]
=
\mathbb{E}_{S}\!\left[\frac{1}{k}\sum_{i\in S}L_i(\theta)\right]
=
\frac{1}{k}\sum_{i=1}^n \Pr[i\in S]\,L_i(\theta)
=
\frac{1}{k}\sum_{i=1}^n \pi_i\,L_i(\theta).
\]
By Lemma~\ref{lem:pi-sum-k}, $\sum_i \pi_i = k$, so this is exactly
$\frac{1}{\sum_i \pi_i}\sum_i \pi_i\,L_i(\theta)=\mathcal{L}_{\mathrm{soft}}(\theta)$.
\end{proof}

\subsection{Training Dataset Sizes}
\label{app:data-sizes}
Table~\ref{tab:data-sizes} summarizes the number of DPO preference pairs produced by our PRISM pre-processing pipeline for each training condition.
Counts are computed after converting PRISM interactions into (prompt, chosen, rejected) pairs and filtering placeholder prompts/responses and survey-only raters without preference data.

 \begin{table}[h]
  \centering
  \small
  \begin{tabular}{lrrl}
    \toprule
    \textbf{Training condition} & \textbf{\# pairs} & \textbf{\# raters} & \textbf{Notes} \\
    \midrule
    Full PRISM & 57{,}750 & 1{,}394 & unweighted \\
    Soft Panel & 57{,}750 & 1{,}394 & weighted by $\pi_i$ \\
    Hard Panel & 5{,}737 & 145 & panel-filtered \\
    US-Rep & 9{,}541 & 230 & PRISM US-Rep subset \\
    \bottomrule
  \end{tabular}
  \caption{Dataset sizes for each training condition. ``\# raters'' counts unique PRISM rater IDs appearing in the prepared pairs for that condition.}
  \label{tab:data-sizes}
\end{table}

\subsection{Locale Composition of Panels}
\label{app:limitations-locale}
Our sortition quotas target U.S.\ demographic proportions, but our default configuration does not filter PRISM raters by \texttt{study\_locale}.
In the resulting hard panel for the US target (panel size $k{=}145$), $56/145$ raters appearing in the training pairs have \texttt{study\_locale=us}; the remainder are mainly from the UK (27) and other locales.
This is a practical compromise to enlarge the feasible rater pool; future work should evaluate panels drawn strictly from U.S.\ locales when sufficient data is available.

\subsection{Panel Feasibility and Quota Satisfaction}
\label{app:panel-feasibility}
We verify that our quota specifications admit feasible panels at our chosen size $k{=}145$ and tolerance $\tau{=}0.1$ (relative bounds), using the LEXIMIN sortition algorithm \cite{flanigan2021fair}.
For the U.S.\ target configuration (seven attributes), we sample a feasible panel of size 145 with zero quota violations from the filtered pool of raters with preference data; across constrained (non-slack) categories, the maximum absolute deviation from target proportions is $4.4$ percentage points.

\subsection{Soft-Panel Selection-Probability Diagnostics}
\label{app:soft-weights}
For the soft panel, we compute the exact selection probabilities $\{\pi_i\}$ from the LEXIMIN-induced lottery and verify the basic correctness properties.
In the census configuration, $\sum_i \pi_i \approx k{=}145$ as required by the linearity of expectation.
After filtering survey-only raters, nearly all selection-probability mass corresponds to raters with preference data; any remaining missing mass reflects filtered or empty preference comparisons rather than invalid weights.

\subsection{Ablations}
\label{app:ablations}
\paragraph{Ordering bias.}
Position bias is a known concern in LLM-as-a-judge settings and can be substantial depending on the prompt and format \cite{shi2025judging}. In our 8B evaluation (5 models; $15{,}000$ rankings), the Pearson correlation between the presented position and the final rank is $\rho=0.047$, and the probability that the response shown at a given position is ranked \#1 ranges from $0.179$ to $0.224$ (a $4.5$ percentage-point spread across positions). We therefore randomize the response order independently for each judge call in all reported results.

\paragraph{Verbosity bias.}
We test whether response length predicts preferences by correlating response length with rank and by checking whether longer responses win implied pairwise comparisons. Across all responses classified with $75{,}000$ ($3000$ questions $\times 5$ judges $\times 5$ models), the correlations between response length and rank are small: $\rho=-0.056$ (characters) and $\rho=-0.044$ (words). In implied pairwise comparisons ($149{,}290$ total), the longer response wins $53.7\%$ of the time, indicating only a mild verbosity preference that does not explain the main ranking.

\subsection{Mallows Model Fit and Stability}
\label{app:mallows}
We fit a Kendall-distance Mallows model by (i) identifying the consensus ranking $\pi_0$ via brute force search over $5!=120$ permutations (for the five-model setting) and (ii) estimating the dispersion parameter $\phi$ by one-dimensional likelihood maximization.
Bootstrap refits ($n{=}1000$) show the consensus is stable (exact-match rate $1.00$), with Hard ranked \#1 and US-Rep ranked \#2 in all bootstrap resamples in the 8B setting (see also \cref{fig:score-grid-8b}).

\subsection{Raw Scoring Outputs}
\label{app:raw-scores}
For reproducibility and to make the plots in \cref{fig:score-grid-8b} numerically explicit, Tables~\ref{tab:raw-borda-copeland}, \ref{tab:raw-kemeny-mallows}, and \ref{tab:raw-bt-pl} report the raw scores computed from our evaluation artifacts.
We also report the corresponding raw scores for the normalized-step 8B ablation in \cref{fig:score-grid-8b-normalized} and Tables~\ref{tab:raw-borda-copeland-8b-normalized} -- \ref{tab:raw-kemeny-mallows-8b-normalized}.
For completeness, we also report the corresponding raw-score tables for Llama-3.2-1B (Tables~\ref{tab:raw-borda-copeland-1b} -- \ref{tab:raw-kemeny-mallows-1b}) and Llama-3.2-3B (Tables~\ref{tab:raw-borda-copeland-3b} -- \ref{tab:raw-kemeny-mallows-3b}).
Listwise scores (Borda, Copeland, Kemeny, PL, Mallows) use all $15{,}000$ judge rankings; BT is fitted on $30{,}000$ pairwise preferences derived from the listwise rankings (8B setting).

\subsection{Training Update Magnitudes}
\label{app:update-magnitude}
Although our objectives are normalized to preserve equal voter voice and a comparable per-step loss scale, the realized update magnitudes can still differ because data distributions differ across conditions. We quantify this by measuring parameter-change magnitude from the base checkpoint after the matched $3538$-step training budget. For Llama-3.1-8B, relative $\ell_2$ deltas are Full $0.00240$, Soft $0.00235$, Hard $0.00275$, US-Rep $0.00291$, indicating larger updates for the smaller hard/US-Rep datasets. We therefore also train normalized-step variants for Hard and US-Rep (Hard: $2200$ steps; US-Rep: $2100$ steps, chosen by decreasing steps in increments of $100$ until the relative $\ell_2$ matched Full/Soft) and find that the ranking order persists even when update magnitudes are matched (see \cref{fig:score-grid-8b-normalized}), though the gaps shrink.

\begin{table}[H]
  \centering
  \small
  \begin{tabular}{lccccc}
    \toprule
    \textbf{Model} & \textbf{Steps} & $\|\Delta\theta\|_2$ & RMS ($\times 10^{-5}$) & Mean abs ($\times 10^{-6}$) & Rel.\ $\ell_2$ \\
    \midrule
    Full PRISM & 3538 & 2.996 & 3.344 & 6.544 & 0.002400 \\
    Soft Panel & 3538 & 2.938 & 3.279 & 6.499 & 0.002354 \\
    Hard Panel & 3538 & 3.435 & 3.834 & 7.070 & 0.002752 \\
    US-Rep & 3538 & 3.626 & 4.047 & 7.483 & 0.002905 \\
    Hard Panel (normalized) & 2200 & 3.021 & 3.371 & 6.104 & 0.002420 \\
    US-Rep (normalized) & 2100 & 2.940 & 3.281 & 5.971 & 0.002355 \\
    \bottomrule
  \end{tabular}
  \caption{Training update magnitudes for Llama-3.1-8B measured as parameter deltas from the base checkpoint. RMS and mean-absolute values are computed over all parameters. ``Rel.\ $\ell_2$'' denotes $\|\Delta\theta\|_2/\|\theta\|_2$ for the base checkpoint.}
  \label{tab:update-magnitude-8b}
\end{table}

\subsection{Reproducibility Details}
\label{app:reproducibility}
All experiments use a single H200 GPU, full-parameter fine-tuning in bf16, and a frozen reference model initialized from the same base checkpoint.
We fix the optimization budget to $3538$ updates for every model and condition, with the same effective batch size (per-device batch size $\times$ gradient accumulation, single GPU), and resample smaller datasets with replacement to reach the same update count.
Evaluation uses identical decoding settings and a shared chat template across all models.
Judging uses GPT-5.2 at temperature $0$ with JSON-formatted outputs and randomized response order for each judge call.
All bootstrap confidence intervals use ranking-level resampling with $n=1000$.

\subsection{Model-Size Scaling}
\label{app:model-size}
We include the 1B and 3B score grids in Figure~\ref{fig:appendix-grids}. The same ordering (Hard $>$ US-Rep $>$ Soft $>$ Full $>$ Base) appears at all sizes, while effect sizes increase with model capacity (see Figure~\ref{fig:model-size-effect} in the main text).

\begin{figure*}[t]
  \centering
  \includegraphics[width=0.49\textwidth]{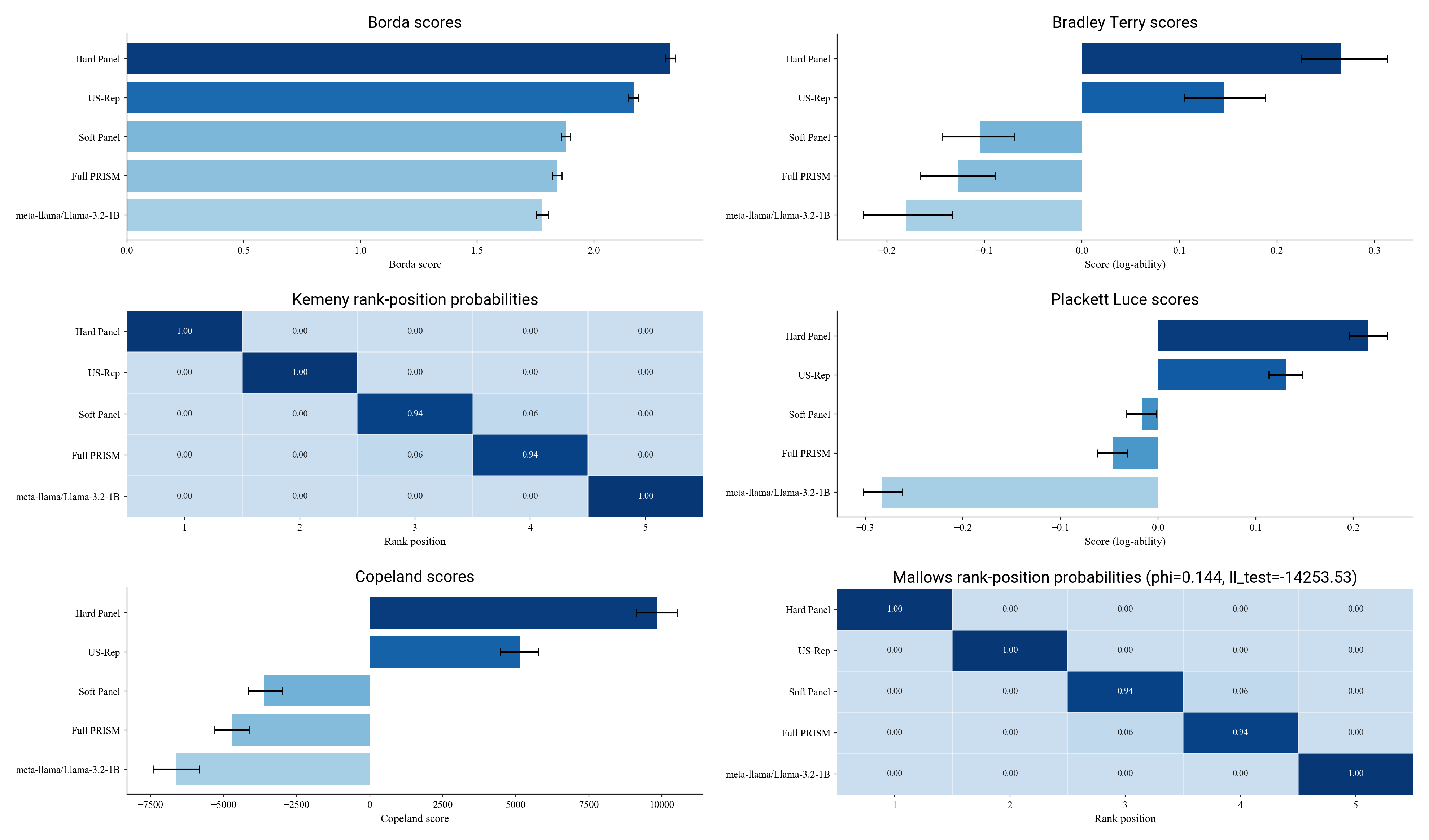}
  \includegraphics[width=0.49\textwidth]{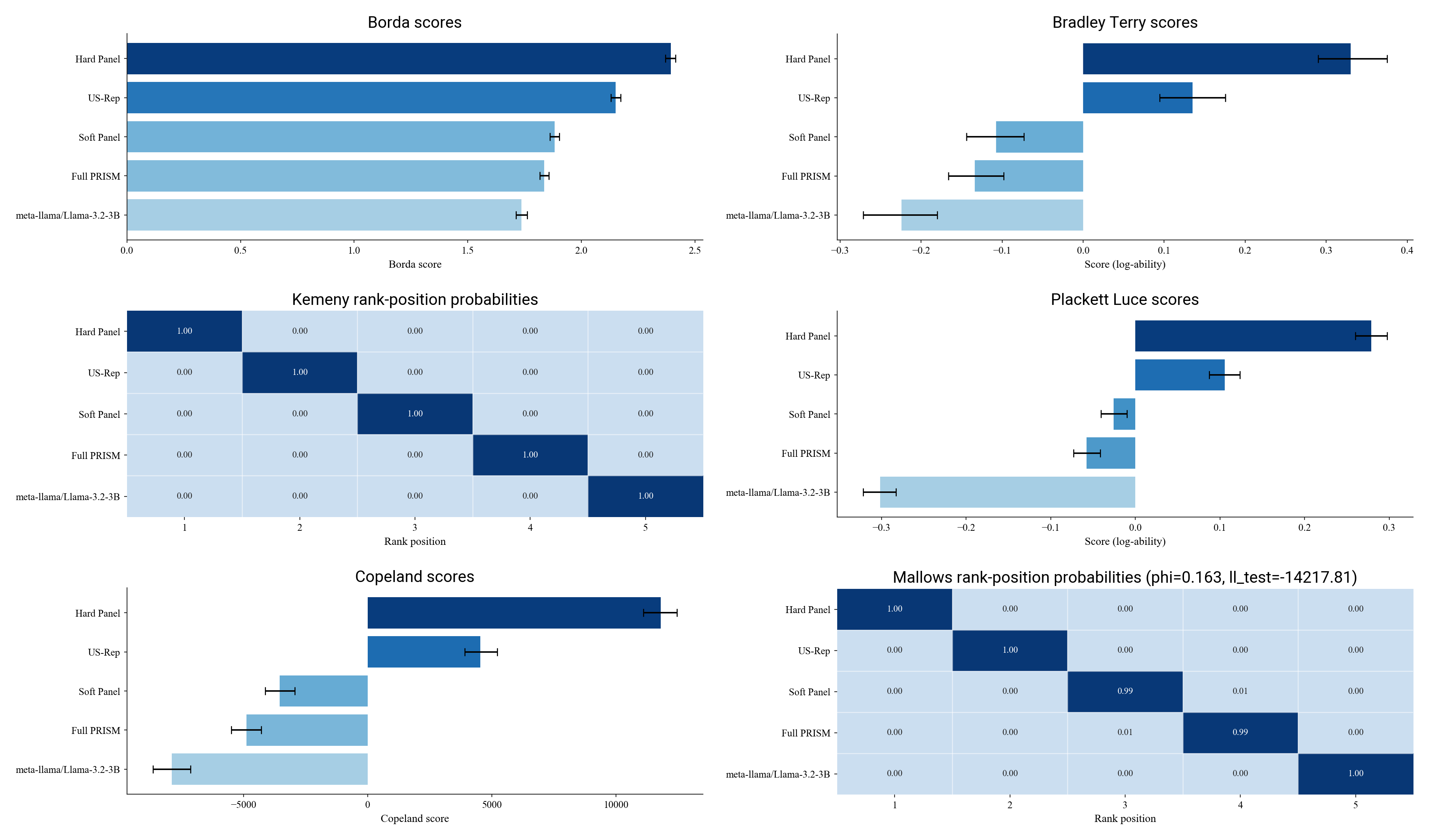}
  \caption{Score grids for Llama-3.2-1B (left) and Llama-3.2-3B (right) under the same evaluation protocol.}
  \label{fig:appendix-grids}
\end{figure*}

\begin{table}[H]
  \centering
  \small
  \begin{tabular}{lccc}
    \toprule
    \textbf{Model} & \textbf{Borda avg} & \textbf{Copeland win rate} & \textbf{Copeland} \\
    \midrule
    Hard Panel & 2.459 [2.437, 2.480] & 0.615 [0.609, 0.620] & 13{,}762 [13{,}100, 14{,}414] \\
    US-Rep & 2.118 [2.097, 2.139] & 0.530 [0.524, 0.535] & 3{,}542 [2{,}900, 4{,}174] \\
    Soft Panel & 2.006 [1.986, 2.025] & 0.501 [0.497, 0.506] & 167 [-416, 752] \\
    Full PRISM & 1.930 [1.909, 1.950] & 0.482 [0.477, 0.487] & -2{,}113 [-2{,}720, -1{,}508] \\
    Base & 1.488 [1.463, 1.514] & 0.372 [0.366, 0.379] & -15{,}358 [-16{,}108, -14{,}576] \\
    \bottomrule
  \end{tabular}
  \caption{Raw Borda and Copeland scores computed from $15{,}000$ list-wise rankings, with 95\% bootstrap confidence intervals (ranking-level bootstrap, $n{=}1000$). ``Borda avg'' is the average per-ranking Borda score (higher is better). ``Copeland'' is wins minus losses in the implied pairwise tournament (higher is better). Numeric values are reported in the released artifacts for each model size.}
  \label{tab:raw-borda-copeland}
\end{table}

\begin{table}[H]
  \centering
  \small
  \begin{tabular}{ll}
    \toprule
    \textbf{Method} & \textbf{Consensus ranking / parameters} \\
    \midrule
    Kemeny-Young & Hard $\succ$ US-Rep $\succ$ Soft $\succ$ Full $\succ$ Base; exact-match (bootstrap) $1.00$ \\
    Mallows (Kendall) & $\pi_0$ matches Kemeny; $\phi=0.219$; $\ell_{\text{test}}=-14107.4$ \\
    \bottomrule
  \end{tabular}
  \caption{Kemeny and Mallows consensus summaries (computed from $15{,}000$ listwise rankings; bootstrap $n{=}1000$ where reported).}
  \label{tab:raw-kemeny-mallows}
\end{table}

\begin{table}[H]
  \centering
  \small
  \begin{tabular}{lcc}
    \toprule
    \textbf{Model} & \textbf{BT score (mean [95\% CI])} & \textbf{PL score (mean [95\% CI])} \\
    \midrule
    Hard Panel & 0.384 [0.345, 0.427] & 0.375 [0.357, 0.393] \\
    US-Rep & 0.102 [0.062, 0.141] & 0.113 [0.097, 0.129] \\
    Soft Panel & 0.017 [-0.018, 0.054] & 0.071 [0.055, 0.087] \\
    Full PRISM & -0.070 [-0.110, -0.034] & 0.026 [0.009, 0.041] \\
    Base & -0.432 [-0.484, -0.381] & -0.584 [-0.604, -0.561] \\
    \bottomrule
  \end{tabular}
  \caption{Bradley--Terry (BT) and Plackett--Luce (PL) scores with 95\% bootstrap confidence intervals (bootstrap $n{=}1000$). BT is fit to the pairwise preferences derived from listwise rankings. Scores are log-abilities up to an additive constant. Numeric values are reported in the released artifacts for each model size.}
  \label{tab:raw-bt-pl}
\end{table}

\paragraph{Llama-3.1-8B normalized-step raw scores.}
\begin{figure}[H]
  \centering
  \includegraphics[width=0.95\textwidth]{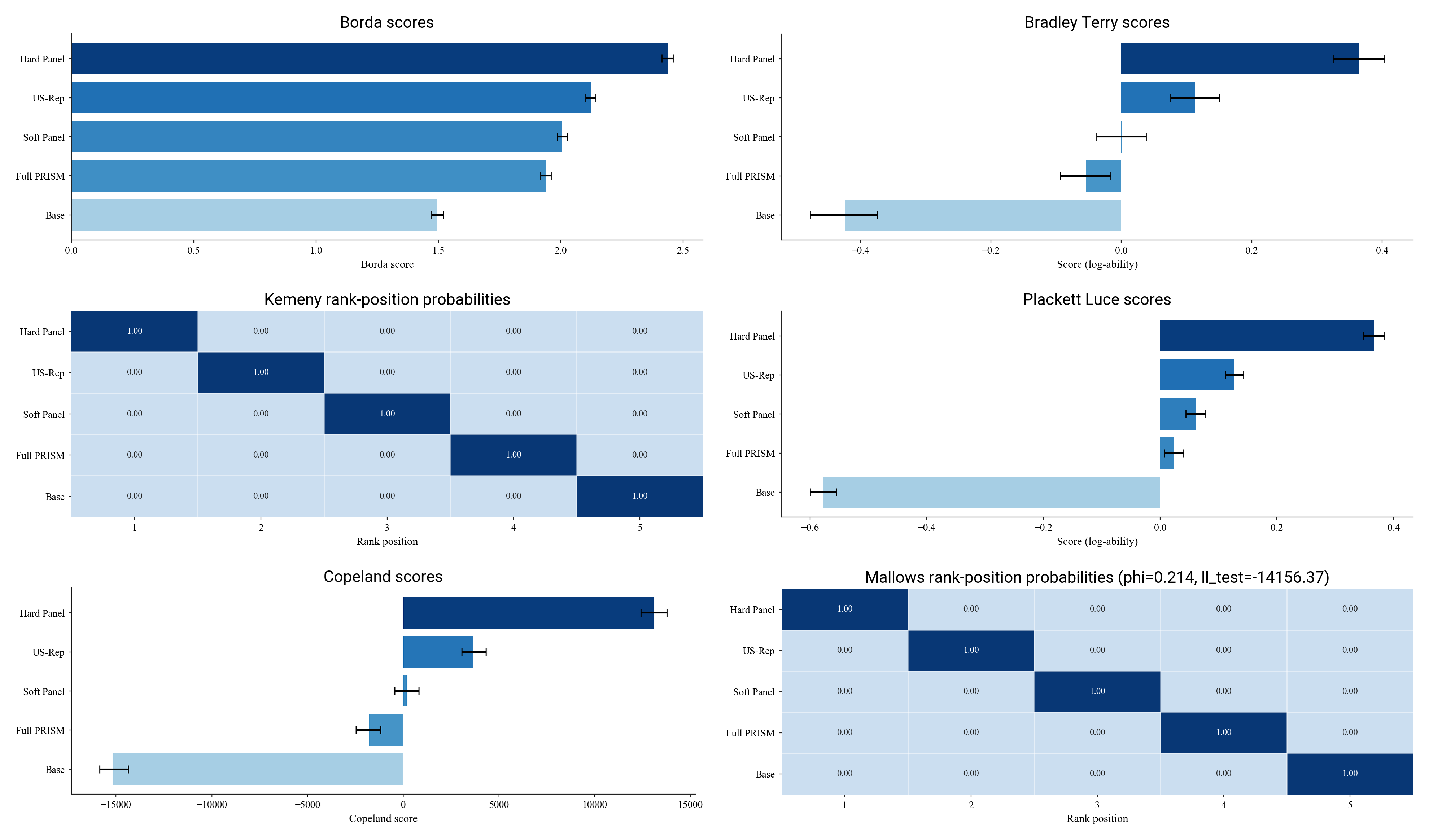}
  \caption{Normalized-step score grid for Llama-3.1-8B, where Hard Panel and US-Rep are retrained for fewer steps (Hard: $2200$, US-Rep: $2100$) to match the update magnitude of Full/Soft.}
  \label{fig:score-grid-8b-normalized}
\end{figure}

\begin{table}[H]
  \centering
  \small
  \begin{tabular}{lccc}
    \toprule
    \textbf{Model} & \textbf{Borda avg} & \textbf{Copeland win rate} & \textbf{Copeland} \\
    \midrule
    Hard Panel & 2.437 [2.414, 2.459] & 0.609 [0.603, 0.615] & 13{,}110 [12{,}416, 13{,}770] \\
    US-Rep & 2.122 [2.102, 2.144] & 0.531 [0.525, 0.536] & 3{,}667 [3{,}052, 4{,}314] \\
    Soft Panel & 2.006 [1.985, 2.027] & 0.502 [0.496, 0.507] & 182 [-445, 818] \\
    Full PRISM & 1.940 [1.918, 1.961] & 0.485 [0.479, 0.490] & -1{,}801 [-2{,}468, -1{,}182] \\
    Base & 1.495 [1.472, 1.521] & 0.374 [0.368, 0.380] & -15{,}158 [-15{,}852, -14{,}358] \\
    \bottomrule
  \end{tabular}
  \caption{Llama-3.1-8B normalized-step Borda and Copeland scores (bootstrap $n{=}1000$).}
  \label{tab:raw-borda-copeland-8b-normalized}
\end{table}

\begin{table}[H]
  \centering
  \small
  \begin{tabular}{lcc}
    \toprule
    \textbf{Model} & \textbf{BT score (mean [95\% CI])} & \textbf{PL score (mean [95\% CI])} \\
    \midrule
    Hard Panel & 0.363 [0.325, 0.404] & 0.366 [0.348, 0.384] \\
    US-Rep & 0.114 [0.076, 0.150] & 0.127 [0.112, 0.143] \\
    Soft Panel & 0.000 [-0.038, 0.038] & 0.061 [0.044, 0.078] \\
    Full PRISM & -0.054 [-0.094, -0.016] & 0.023 [0.007, 0.040] \\
    Base & -0.423 [-0.477, -0.374] & -0.578 [-0.600, -0.555] \\
    \bottomrule
  \end{tabular}
  \caption{Llama-3.1-8B normalized-step BT and PL scores (bootstrap $n{=}1000$).}
  \label{tab:raw-bt-pl-8b-normalized}
\end{table}

\begin{table}[H]
  \centering
  \small
  \begin{tabular}{ll}
    \toprule
    \textbf{Method} & \textbf{Consensus ranking / parameters} \\
    \midrule
    Kemeny-Young & Hard $\succ$ US-Rep $\succ$ Soft $\succ$ Full $\succ$ Base; exact-match (bootstrap) $1.00$ \\
    Mallows (Kendall) & $\pi_0$ matches Kemeny; $\phi=0.214$; $\ell_{\text{test}}=-14156.4$ \\
    \bottomrule
  \end{tabular}
  \caption{Llama-3.1-8B normalized-step Kemeny and Mallows summaries (bootstrap $n{=}1000$ where reported).}
  \label{tab:raw-kemeny-mallows-8b-normalized}
\end{table}

\paragraph{Llama-3.2-1B raw scores.}
\begin{table}[H]
  \centering
  \small
  \begin{tabular}{lcc}
    \toprule
    \textbf{Model} & \textbf{Borda avg (mean [95\% CI])} & \textbf{Copeland (mean [95\% CI])} \\
    \midrule
    Hard Panel & 2.328 [2.305, 2.350] & 9{,}836.294 [9{,}137.950, 10{,}514.050] \\
    US-Rep & 2.171 [2.149, 2.192] & 5{,}125.712 [4{,}469.900, 5{,}768.400] \\
    Soft Panel & 1.880 [1.861, 1.900] & -3{,}600.814 [-4{,}156.050, -2{,}987.950] \\
    Full PRISM & 1.842 [1.823, 1.862] & -4{,}729.654 [-5{,}302.100, -4{,}127.950] \\
    Base & 1.779 [1.753, 1.805] & -6{,}631.538 [-7{,}416.050, -5{,}839.750] \\
    \bottomrule
  \end{tabular}
  \caption{Llama-3.2-1B Borda and Copeland scores (bootstrap $n{=}1000$).}
  \label{tab:raw-borda-copeland-1b}
\end{table}

\begin{table}[H]
  \centering
  \small
  \begin{tabular}{lcc}
    \toprule
    \textbf{Model} & \textbf{BT score (mean [95\% CI])} & \textbf{PL score (mean [95\% CI])} \\
    \midrule
    Hard Panel & 0.267 [0.225, 0.313] & 0.215 [0.196, 0.235] \\
    US-Rep & 0.146 [0.105, 0.188] & 0.132 [0.113, 0.148] \\
    Soft Panel & -0.106 [-0.143, -0.069] & -0.017 [-0.032, -0.002] \\
    Full PRISM & -0.127 [-0.166, -0.089] & -0.047 [-0.062, -0.032] \\
    Base & -0.180 [-0.224, -0.133] & -0.283 [-0.302, -0.262] \\
    \bottomrule
  \end{tabular}
  \caption{Llama-3.2-1B BT and PL scores (bootstrap $n{=}1000$).}
  \label{tab:raw-bt-pl-1b}
\end{table}

\begin{table}[H]
  \centering
  \small
  \begin{tabular}{ll}
    \toprule
    \textbf{Method} & \textbf{Consensus ranking / parameters} \\
    \midrule
    Kemeny-Young & Hard $\succ$ US-Rep $\succ$ Soft $\succ$ Full $\succ$ Base; exact-match $0.94$ \\
    Mallows (Kendall) & $\pi_0$ matches Kemeny; $\phi=0.144$; $\ell_{\text{test}}=-14253.5$ \\
    \bottomrule
  \end{tabular}
  \caption{Llama-3.2-1B Kemeny and Mallows summaries (bootstrap $n{=}1000$ where reported).}
  \label{tab:raw-kemeny-mallows-1b}
\end{table}

\paragraph{Llama-3.2-3B raw scores.}
\begin{table}[H]
  \centering
  \small
  \begin{tabular}{lcc}
    \toprule
    \textbf{Model} & \textbf{Borda avg (mean [95\% CI])} & \textbf{Copeland (mean [95\% CI])} \\
    \midrule
    Hard Panel & 2.393 [2.370, 2.415] & 11{,}797.238 [11{,}102.000, 12{,}446.350] \\
    US-Rep & 2.151 [2.130, 2.174] & 4{,}537.120 [3{,}905.850, 5{,}208.400] \\
    Soft Panel & 1.882 [1.862, 1.902] & -3{,}542.388 [-4{,}134.000, -2{,}931.950] \\
    Full PRISM & 1.837 [1.817, 1.857] & -4{,}892.736 [-5{,}500.100, -4{,}283.950] \\
    Base & 1.737 [1.712, 1.762] & -7{,}899.234 [-8{,}642.050, -7{,}129.950] \\
    \bottomrule
  \end{tabular}
  \caption{Llama-3.2-3B Borda and Copeland scores (bootstrap $n{=}1000$).}
  \label{tab:raw-borda-copeland-3b}
\end{table}

\begin{table}[H]
  \centering
  \small
  \begin{tabular}{lcc}
    \toprule
    \textbf{Model} & \textbf{BT score (mean [95\% CI])} & \textbf{PL score (mean [95\% CI])} \\
    \midrule
    Hard Panel & 0.331 [0.290, 0.375] & 0.279 [0.260, 0.297] \\
    US-Rep & 0.135 [0.094, 0.175] & 0.106 [0.087, 0.123] \\
    Soft Panel & -0.108 [-0.144, -0.073] & -0.026 [-0.041, -0.010] \\
    Full PRISM & -0.133 [-0.166, -0.098] & -0.057 [-0.073, -0.041] \\
    Base & -0.225 [-0.271, -0.180] & -0.302 [-0.322, -0.283] \\
    \bottomrule
  \end{tabular}
  \caption{Llama-3.2-3B BT and PL scores (bootstrap $n{=}1000$).}
  \label{tab:raw-bt-pl-3b}
\end{table}

\begin{table}[H]
  \centering
  \small
  \begin{tabular}{ll}
    \toprule
    \textbf{Method} & \textbf{Consensus ranking / parameters} \\
    \midrule
    Kemeny-Young & Hard $\succ$ US-Rep $\succ$ Soft $\succ$ Full $\succ$ Base; exact-match $1.00$ \\
    Mallows (Kendall) & $\pi_0$ matches Kemeny; $\phi=0.163$; $\ell_{\text{test}}=-14217.8$ \\
    \bottomrule
  \end{tabular}
  \caption{Llama-3.2-3B Kemeny and Mallows summaries (bootstrap $n{=}1000$ where reported).}
  \label{tab:raw-kemeny-mallows-3b}
\end{table}

\end{document}